\def\year{2015}
\documentclass[letterpaper]{article}
\usepackage{aaai}

\pdfoutput=1     

\usepackage{url}

\usepackage{times}
\usepackage{helvet}
\usepackage{courier}
\usepackage{amssymb}
\usepackage{amsmath, amsthm}
\usepackage{color}
\usepackage{graphicx}
\usepackage{verbatim}
\usepackage{tabularx}
\usepackage{amsfonts,amsmath,amsthm,array,amssymb}
\usepackage{graphicx}      
\usepackage{times}
\usepackage{epstopdf}
\usepackage{url}
\usepackage{algorithm, algorithmicx, algpseudocode}
\usepackage{subfig}

\usepackage{multirow}

\newtheorem{theorem}{Theorem}

\newtheorem{proposition}{Proposition}

\begin{document}
%
\title{Pattern Decomposition with Complex Combinatorial Constraints:\\ Application to Materials Discovery}
\author{
Stefano Ermon \\  Computer Science Department 
\\ Stanford University \\ ermon@cs.stanford.edu
\And
Ronan Le Bras \\ Department of  Computer Science 
\\ Cornell University \\ lebras@cs.cornell.edu
\And
Santosh K. Suram, John M. Gregoire \\  Joint Center for Artificial Photosynthesis
\\ California Institute of Technology
\\ \{sksuram, gregoire\}@caltech.edu 
\AND
Carla P. Gomes, Bart Selman \\  Department of Computer Science 
\\ Cornell University \\ \{gomes,selman\}@cs.cornell.edu  
\And
Robert B. van Dover \\  Department of Materials Science and Engineering
\\ Cornell University \\ rbv2@cornell.edu
}
\maketitle
\begin{abstract}
\begin{quote}
Identifying important components or factors in large amounts of noisy data is a key problem in machine learning and data mining. Motivated by a pattern decomposition problem in materials discovery, aimed at discovering new materials for renewable energy, e.g. for fuel and solar cells, we introduce CombiFD, a framework for factor based pattern decomposition that allows the incorporation of a-priori knowledge as constraints, including complex combinatorial constraints.  In addition, we propose a new pattern decomposition algorithm, called AMIQO, based on solving a sequence of (mixed-integer) quadratic programs. Our approach considerably outperforms the state of the art on the materials discovery problem, scaling to larger datasets and recovering more precise and physically meaningful decompositions. We also show the effectiveness of our approach for enforcing background knowledge on other application domains. 
\end{quote}
\end{abstract}

\section{Introduction}
\label{introduction}
In recent years, we have seen an enormous growth in data generation rates in many fields
of science~\cite{halevy2009unreasonable}. For instance, in combinatorial materials discovery, scientists search for new materials with desirable properties by obtaining measurements on hundreds of samples in a single batch
experiment using a composition spread technique~\cite{ginley,narasimhan2007combinatorial}. 
Providing computational tools for automatically analyzing and for determining the structure of the materials formed in a composition spread 
is an important and exciting direction in the emerging field of Computational Sustainability~\cite{gomes2009computational}. 
For example, this approach has been successfully applied to speed up the discovery of new materials with improved catalytic activity for fuel cell applications~\cite{gregoireptta,van1998discovery} and of oxygen evolution catalysts with applications to solar fuel generation \cite{Haber2014Discovering}. Long-term solutions to several sustainability issues in energy and transportation are likely to come from break-through innovations in materials~\cite{white2012materials}, and computer science can play a role 
and provide support 
to accelerate new materials discovery~\cite{lebras2011constraint}.

To accelerate the discovery of new materials, materials scientists have developed high-throughput experimental procedures to create composition-spread libraries of new materials, a process that can be intuitively understood as generating an enormous number of compounds in a single experiment by mixing different amounts of a small number of basic elements~\cite{takeuchi2002combinatorial}. The promising libraries are then characterized using X-ray diffraction 
to determine the underlying crystal structure and composition. Specifically, a set of X-ray diffraction signals are sampled at $n$ different material compositions, each one corresponding to a different mixture of some basic elements.
A key problem in materials discovery is called the \emph{phase map identification problem}, defined as finding $k \ll n$ basic phases (basis functions that change gradually with composition, in terms of structure and intensity), such that all the $n$ X-ray measurements can be explained as a mixture of the $k$ basic phases.
The decomposition is subject to physical constraints that govern the underlying crystallographic process.

The \emph{phase map identification} is effectively a source separation or spectral unmixing problem~\cite{berry2007algorithms} where the sources are the $k$ non-negative, basic x-ray diffraction signals and each observation is a non-negative combination of these $k$ sources. 
Therefore, a standard approach from the literature is non-negative matrix factorization (NMF)~\cite{long09}. 
Nevertheless, this approach overlooks the physical constraints from the crystal formation. For example, it does not guarantee connectivity of the ``phase regions'' in the phase map, nor can it handle basis patterns that are slightly changing (``shifting'') as the crystal lattice constants change. Recent development \cite{kusne2014fly}, while not enforcing these constraints \emph{per se}, have shown to be resilient to peak shifting for example.
To obtain physically meaningful decompositions, researchers~\cite{lebras2011constraint,ermon2012smt} have looked at constraint programming formulations that can incorporate all the necessary constraints. The down side of these approaches is that they are fully discrete (they require a discretization of the data through peak detection) and they cannot directly deal with continuous measurement data. On the other side of the spectrum with respect to NMF, the unsupervised nature is lost, and scalability becomes an issue as, for example, in a fully discrete problem there is no notion of gradient anymore. In addition, these approaches are not robust to the presence of noise in the data, as noise considerably impacts the efficiency of their filtering and propagation mechanisms.

In this work, we aim to achieve the best of both worlds by bridging the previous approaches and providing a new hybrid formulation, where we integrate additional domain knowledge as additional constraints into the basic NMF approach. 
We introduce CombiFD, a novel pattern decomposition framework that allows the specification of very general constraints. These include constraints used with some matrix factorization and clustering approaches (such as non-negativity or partial labeling information) as well as more general ones that require a richer representation language, powerful enough to capture more complex, combinatorial dependencies. For example, we show 
how to encode complex a-priori scientific domain knowledge by specifying combinatorial dependencies on the variables imposed by physical laws. We also propose a novel solution technique called AMIQO (\textbf{A}lternating \textbf{M}ixed \textbf{I}nteger \textbf{Q}uadratic \textbf{O}ptimization), that involves the solution of a sequence of (mixed-integer) quadratic programs.
Overall, our constrained factorization algorithm clearly outperforms previous approaches: it scales to large real world datasets, and recovers physically meaningful and significantly more accurate interpretations of the data when prior knowledge is taken into account.

\section{Framework}
\label{main}
Given $n$ data points $a_i \in \mathbb{R}^{m}$, each one represented by an $m$-dimensional vector of real-valued features, we represent the input data compactly as a matrix $A \in \mathbb{R}^{m \times n}$, each column corresponding to a data point and each row to a feature. In the context of the \emph{phase map identification}, $A$ corresponds to the $n$ observed X-ray diffraction patterns, each of them as a vector of $m$ scattering intensity values.

We are interested in low-dimensional representations which can approximate the input data $A$ by identifying its essential components or factors. Namely, for a given number $k$ of basic phases, we want to approximate $A$ as $A \approx WH $, where $W\in \mathbb{R}^{m \times k}$ represents $k$ basic phases (or phase patterns) and $H\in \mathbb{R}^{k \times n}$ the combination coefficients at each data point. 

This problem belongs to the family of low-rank approximation problems, an important research theme in machine learning and data mining with numerous applications, including source separation, denoising, compression, and dimensionality reduction~\cite{berry2007algorithms}. 

\subsection{Low-Rank Approximation}
Given a non-negative matrix $A \in \mathbb{R}^{m \times n}$ and a desired rank $k < \min(n,m)$, we seek matrices $W\in \mathbb{R}^{m \times k} $ and $H\in \mathbb{R}^{k \times n}$, $H,W \geq 0$ that give the best low rank approximation of $A$, i.e.
$
A \approx WH 
$. Typically this is formulated as the following optimization problem:
\begin{equation}
\label{nmf_opt}
\min_{W,H} || A-WH||_2 
\end{equation}
where $W \in \mathbb{R}^{m \times k}$ is a matrix of \emph{basis vectors} or \emph{patterns} and $H \in \mathbb{R}^{k \times n}$ the \emph{coefficient} matrix. The symbol $|| \cdot||_2$ indicates (entry-wise) Frobenius norm. 

This basic problem can be solved by the so-called truncated Singular Value Decomposition (SVD) approach, which produces the best approximation in terms of squared distance. It can be computed efficiently and robustly, obtaining a representation where data points can be interpreted as linear combinations of a small number of basis vectors. 

In many applications input data are non-negative, for instance representing color intensities, word counts ~\cite{tjioe2008using} or x-ray scattering intensities in our motivating application. The basis vectors computed with an SVD, however, are generally not guaranteed to be non-negative, and this leads to an undesirable \emph{lack of interpretability}. For example, it is not possible to interpret an image as the superposition of simpler patches, or 
an X-ray diffraction pattern as the composition of several basic compounds 
when obtaining negative values for some of the basis vectors or the coefficients. To overcome this limitation, researchers have introduced the NMF approach, which explicitly enforces non-negativity constraints on the basis vectors and coefficients. 

While non-negativity is a very common constraint in many domains, in some applications we have additional valuable \emph{a priori} information on the features (each feature corresponds to a row of $A$ and $W$). For instance, we might also know \emph{a priori} an upper bound on the value of some features, or that two Boolean features are incompatible, or that non-negativity only holds for a subset of the features. In particular, in our materials science domain basis, vectors (patterns) correspond to chemical compounds and their underlying crystal structures. For chemical systems in thermodynamic equilibrium, the compositional variation of the concentration and lattice parameters of each compound follow well-defined rules, from which constraints on the coexistence and variation of basis patterns can be defined~\cite{lebras2011constraint,ermon2012smt}. 
While some constraints such as non-negativity can be individually enforced in some approaches, others such as connectivity and the complex rules defining shifting basis patterns (see below for details) have not been considered before. To the best of our knowledge, there is no general factor analysis framework that can handle the \emph{combination} of all these constraints. 
This motivates the definition of the following general pattern decomposition subject to combinatorial constraints problem.
\subsection{CombiFD: Pattern Decomposition with Combinatorial Constraints}
Given a (general) matrix $A \in \mathbb{R}^{m \times n}$ and a desired rank $k < \min(n,m)$ we seek matrices $W\in \mathbb{R}^{m \times k} $ and $H\in \mathbb{R}^{k \times n}$ that minimize $||A-WH||_{p}$ where $p \in \{1,2\}$ and $||\cdot||_p$ is an entry-wise norm (e.g. $p=2$ corresponds to the Frobenius norm).
Moreover, the factors need to satisfy an additional set of $J$ linear inequality constraints, possibly requiring binary or integer variables. This is formalized as the following optimization problem:
\begin{equation}
\label{cnmf_def}
\begin{aligned}
& \underset{W,H,x,b}{\text{minimize}}
& & || A-WH||_{p} = f(W,H) \\
& \text{subject to}
& &  C [\text{vec}(W),\text{vec}(H),x,b]^T\leq d \\
&&&  b_i \in \{0,1\}, \; i \in [1,N].
\end{aligned}
\end{equation}
where $x\in \mathbb{R}^M$ is a vector of additional real-valued variables, $b$ is a vector of $N$ binary variables, $d \in \mathbb{R}^J$, $C\in \mathbb{R}^{J \times (mk + nk + M +N)}$  and $\text{vec}(\cdot)$ denotes vectorization (stacking a matrix into a vector). That is, we have $J$ linear inequalities involving the entries of $W$, $H$, $x$ and $b$, with coefficients given by $C$ and right-hand side given by $d$. As Integer Linear Programming is well known to be NP-complete, very general constraints can be encoded by appropriately choosing $C$ and $d$.
These additional (combinatorial) constrains are extremely useful to encode prior knowledge we might have about the domain. These include non-negativity ($W,H \geq 0$), upper bounds ($W_{i,j}\leq u$), sparsity (see Example 1), semi-supervised clustering (see Example 2) as well as many others. Other examples of intricate constraints can be found in the experimental section below.

\textbf{Example 1: $L_0$ sparsity}
Suppose we want to explicitly formulate a sparsity constraint on the coefficient matrix. That is, we want to find $W,H$ such that $A \approx WH$ and each column of $H$ has at most $S$ non-zero entries, i.e. $||H_i||_0 \leq S$ where $H_i$ is the $i$-th column of $H$. 
For instance, in a topic modeling application where the basis vectors correspond to topics, this constraint ensures that each document can have at most $S$ topics.
This can be encoded as follows:
\begin{equation}
\label{cnmf_sparse_def}
\begin{aligned}
&\underset{W,H,b}{\text{min}}
& &||A-WH||_{2} \\
&\text{s.t.}
&& b_{i,j} \geq h_{i,j}, \sum_j h_{i,j} = 1,  \sum_i b_{i,j} \leq S \\
&&& b_{i,j} \in \{0,1\} \; i \in [1,k], j \in [1,n]
\end{aligned}
\end{equation}
which can be easily rewritten more compactly in the form (\ref{cnmf_def}) by selecting appropriate $C$ and $d$ (in this case, $M=0$ and $N=kn$).

\textbf{Example 2: Semi-Supervised Clustering}
As an example, in a semi-supervised clustering problem, we can easily include partial labeling information as constraints on $H$, e.g. enforcing Must-Link or Cannot-Link constraints~\cite{liu2010non,basu2008constrained,choo2013utopian}. Suppose we have prior information on $P_{ML}$ pairs of data points $ML=\{(i_1,j_1), \cdots, (i_{P_{ML}},j_{P_{ML}})\}$ that are known to belong to the same cluster (Must-Link)
, and $p_{CL}$ pairs of data points $CL=\{(i_1,j_1), \cdots, (i_{p_{CL}},j_{p_{CL}})\}$ that are known not to belong to the same cluster (Cannot-Link).
We encode this prior knowledge into our CombiFD framework as follows:
\begin{equation}
\label{cnmf_semi_supervised}
\begin{aligned}
&\underset{W,H,b}{\text{min}}
& &||A-WH||_{2} \\
&\text{s.t.}
&& W,H \geq 0 , \sum\nolimits_j h_{i,j} = 1, \sum\nolimits_i b_{i,j} \leq S \\
&&& b_{i,j} \geq h_{i,j},\; b_{i,j} \in \{0,1\} \; i \in [1,k], j \in [1,n] \\
&&& b_{i,i_s} = b_{i,j_s} \; i \in [1,k], (i_s,j_s) \in ML \\
&&& b_{i,i_s} + b_{i,j_s} \leq 1 \; i \in [1,k], (i_s,j_s) \in CL
\end{aligned}
\end{equation}

\subsection{Related work}

Many low rank approximation schemes are available, including QR decomposition, Independent Component Analysis, truncated Singular Value Decomposition, and Non-negative Matrix Factorization~\cite{berry2007algorithms}. 

While these basic methods are unsupervised, there is a growing interest in incorporating prior knowledge or user guidance into these frameworks~\cite{zhi2013clustering}. 
For example, in semi-supervised clustering applications, user guidance is often given by partial labeling information, which can be incorporated using hard constraints ~\cite{liu2010non,basu2008constrained,choo2013utopian}. Typical constraints used in this case are Must-Link and Cannot-Link, enforcing that two data points must or cannot be in the same cluster, respectively. For example,~\cite{hossain2010unifying} present an integrated framework for clustering non-homogenous data, and show how to turn Must-Link and Cannot-Link constraints into dependent and disparate clustering problems. Recently, researchers have also considered interactive matrix factorization schemes for topic modeling that can take into account user feedback on the quality of the decomposition~\cite{choo2013utopian} (topic refinement, merging or splitting)  and semi-supervised NMF with label information as constraints~\cite{liu2010non}.  Constraint clustering~\cite{basu2008constrained} is another example of this approach. Alternatively, regularizations or penalty terms are also used to obtain solutions with certain desired properties such as sparsity~\cite{hoyer2004non,cai2011graph},  
convexity~\cite{ding2010convex}, temporal structural properties and shift invariances~\cite{smaragdis2004non,smaragdis2008sparse}.

Most of the work in the literature is however confined to a single type of constraints or simple conjunctions, which limits their usability. With CombiFD, we propose a new approach for finding a low dimensional approximation of some data (decomposition into basic patterns) that is able to incorporate not only existing types of constraints but also more complex logical combinations.

\section{Constrained Factorization Algorithm}
\label{algorithm}
Solving the general CombiFD optimization problem (\ref{cnmf_def}) is challenging for two reasons. First, the objective function is not convex, hence minimization is difficult even in the presence of simple non-negativity constraints~\cite{berry2007algorithms}. Second, we are allowing a very expressive class of constraints, which can potentially specify very complex, intricate dependencies among the variables. Unfortunately, general nonconvex mixed-integer non-linear programming has not seen as much progress as their linear (MILP) and quadratic (MIQP) counterparts, and most approaches are either application specific or do not scale well. Even in the presence of simple constraints (as it is the case for NMF), the problem is rarely solved to optimality and in practice heuristic approaches are used. Yet, simple heuristic approaches such as multiplicative update rules~\cite{lee1999learning}, which is one of the most widely used algorithms, do not apply to our case due to the integer variables. Similarly, projected gradient techniques cannot be directly applied here~\cite{lin2007projected}.
We thereofore introduce a new approximate technique called AMIQO which exploits the special structure of the problem and takes full advantage of advanced MIQP optimization techniques from the OR literature. 

To solve the general CombiFD optimization problem (\ref{cnmf_def}), we introduce AMIQO (\textbf{A}lternating \textbf{M}ixed \textbf{I}nteger \textbf{Q}uadratic \textbf{O}ptimization) with pseudocode reported as Algorithm \ref{algo}. AMIQO is an iterative two-block coordinate descent procedure enhanced with sophisticated combinatorial optimization techniques beyond the standard convex optimization methods. In fact, the key advantage of our CombiFD framework is that for either $H$ or $W$ fixed, problem (\ref{cnmf_def}) is a mixed-integer quadratic program.\footnote{For $p=1$, it simplifies to a mixed-integer linear program.} Mixed-integer quadratic programs have been widely studied in the operations research literature, and we can leverage a wide range of techniques that have been developed and are implemented in state-of-the-art mixed-integer quadratic programming (MIQP) solvers such as IBM CPLEX. 
These integer programs do not have to be solved to optimality, and it is sufficient to improve the objective function with respect to the factorization found at the previous step (which is used to warm-start the search).
Notice that when there are no binary variables ($N=0$), the optimization problems in the inner loop of AMIQO correspond to standard quadratic programs that can be solved efficiently, even in the presence of (linear) constraints in the form (\ref{cnmf_def}), which are more general than non-negativity.
AMIQO is inspired by the seminal work of Paatero and Tapper who initially proposed the use of a block coordinate descent procedure for NMF~\cite{paatero1994positive}, and was later followed upon by  a number of researchers, including an unconstrained least squares version~\cite{berry2007algorithms}, and solution techniques based on projected gradient descent~\cite{lin2007projected}, Quasi-Newton~\cite{kim2007fast}, and Active-set~\cite{kim2008nonnegative}. However, our approach is novel in that it uses a mixed integer solver in each coordinate descent step, and is the only one that can take into account combinatorial constraints, guaranteeing feasibility of the solution at every iteration even in the presence of intricate combinatorial constraints.

\begin{algorithm}
\begin{algorithmic}
\State Find feasible $W^0,H^0,x^0,b^0$ for (\ref{cnmf_def}) \Comment{Use MIP solver}
\For {$j=0, \cdots, t-1$}
\State $W^{j+1},\tilde{x}^{j+1},\tilde{b}^{j+1} \leftarrow \arg \min_{W,x,b} f(W,H^{j})$ s.t. (\ref{cnmf_def}) and $H=H^j$ \Comment{Use MIQP solver}
\State $H^{j+1},x^{j+1},b^{j+1} \leftarrow \arg \min_{H,x,b} f(W^{j+1},H)$ s.t. (\ref{cnmf_def}) and $W=W^{j+1}$ \Comment{Use MIQP solver}
\EndFor
\State \Return $W^t,H^t$
\end{algorithmic}
\caption{AMIQO}
\label{algo}
\end{algorithm}
We summarize the properties of AMIQO with the following proposition:
\begin{proposition}\label{thProperties}
Let $W^j,H^j$ be as in Algorithm \ref{algo}. If (\ref{cnmf_def}) is feasible, the following two properties hold:
	1) For all $j$, $0 \leq j \leq t$, the optimization problems in the inner loop of the algorithm are feasible and $(W^j,H^j,x^j,b^j)$ is feasible for (\ref{cnmf_def}). 
	2) The objective function $|| A-W^jH^j||_{p}$ is monotonically non-increasing, i.e. $|| A-W^jH^j||_{p} \geq || A-W^{j+1}H^{j+1}||_{p}$
\end{proposition}

\begin{theorem}\label{thKmeanEquiv}
AMIQO run on CombiFD problem (\ref{cnmf_semi_supervised}) with $S=1, ML=CL=\emptyset$ is equivalent to $k$-means.
\end{theorem}
\begin{proof}
See Appendix for both proofs.
\end{proof}

Although the objective function is monotonically non-increasing, AMIQO is not guaranteed to converge to a global minimum. This is consistent with the hardness of problem (\ref{cnmf_def}). More specifically, the quality of the final solution found might depend on the initialization of $W^0$ and $H^0$. This issue is common to other standard matrix factorization algorithms, and several heuristic initialization schemes have been proposed to mitigate the issue~\cite{albright2006algorithms}. 
Nevertheless, in our experimental evaluation, the initialization did not play a major role, and we typically converged to the same solution, regardless of the initial conditions.

\section{Experiments -- Encoding domain knowledge as additional constraints}
\label{experiments}
In order to show the generality of our approach, we provide experimental results on three application domains, with increasingly more complex constraints capturing a-priori domain knowledge. We start with semi-supervised clustering with partial labeling information, i.e. simple Must-Link and Cannot-Link constraints. We then consider another clustering problem, where we include more complex logical constraints on the features, describing higher-level biological knowledge. Finally we consider our motivating application: the phase map identification problem.

\subsection{Clustering with partial labeling information}
NMF has become a popular approach for clustering, with application domains ranging from document clustering~\cite{shahnaz2006document} and graph clustering~\cite{kuang2012symmetric} to gene expression data clustering~\cite{tjioe2008using}. Cluster membership is determined by the coefficient matrix $H$, which reflects how each data point decomposes into the basis vectors. 

There are several ways to obtain hard clustering assignments (binary indicators) from the coefficient matrix $H$ (real valued). We follow 
~\cite{ding2008equivalence}
and normalize the matrix $H$ so that the entries can be interpreted as the posterior probability $p(c_s|d_j)$ that a data point $j$ belongs to cluster $s$. 
Specifically, we let $D_W=diag(\mathbf{1}^TW)$ and estimate the cluster membership probability as $p(c_s|d_j) \propto [D_W H]_{sj}$. As a result, a data point $j$ is assigned a cluster $s^*(j)$ such that $s^*(j) = argmax_s [D_W H]_{sj}$.

We consider a semi-supervised clustering task where we assume to have some prior information on the labels (equivalently, on the cluster assignment) of a subset of datapoints. Specifically, we assume to have information about pairs of data points, which should either belong to the same cluster (Must-Link) or not (Cannot-Link). This information is obtained using standard labeled datasets from the UCI repository~\cite{BacheLichman2013} for which a ground truth clustering is known. To generate various amounts of prior knowledge, we randomly select $P$ pairs of data points, using their labels to specify a MustLink or CannotLink constraint. 

We compare our CombiFD formulation of the problem (\ref{cnmf_semi_supervised}) with two previous approaches from the literature: CNMF (Constrained NMF)~\cite{liu2010non} and NMFS (NMF-based semi supervised clustering)~\cite{li2007solving}. The first approach is based on enforcing non combinatorial constraints on $H$, while the second approach captures the ML and CL constraints using penalty terms in a modified objective function which is then approximately minimized using multiplicative updates.

We report in Figure \ref{MLCLAccuracy} the accuracy obtained using these methods as a function of the amount of supervision, i.e. the number of Must-Link or Cannot-Link constraints used. 
Accuracy is defined as in ~\cite{liu2010non} and corresponds to $AC=1/n \cdot max_{\sigma} \sum_{i=1}^{k}|r_i \cap c_{\sigma_i}|$, where $\sigma:1..k \rightarrow 1..k$ is a bijection mapping clusters $r_i$ to ground-truth classes $c_j$. Namely, each cluster is assigned a label such that the labeling best matches the ground truth labels. Note that the accuracy can be computed efficiently using, for example, the Hungarian algorithm.

Results are averaged over $100$ runs. We see that CombiFD significantly outperforms the competing techniques across all levels of prior knowledge. Intuitively, this is because by properly taking into account the combinatorial nature of the problem, CombiFD can automatically make logically sound inferences about the data: for example, we can take into account transitive closure (i.e., if $a$ and $b$ must link, and $b$ and $c$ must link, then $a$ and $c$ must link as well) and other logical implications. The deeper reasoning power and greater accuracy provided by AMIQO however involves a small computational overhead, with a typical runtime in the order of a couple of minutes for AMIQO versus a few tens of seconds for the competing techniques.
\begin{figure*}[tb]
\subfloat[Iris]{\label{grIris}\includegraphics[trim=0 0 00 0, clip, width=0.22\textwidth]{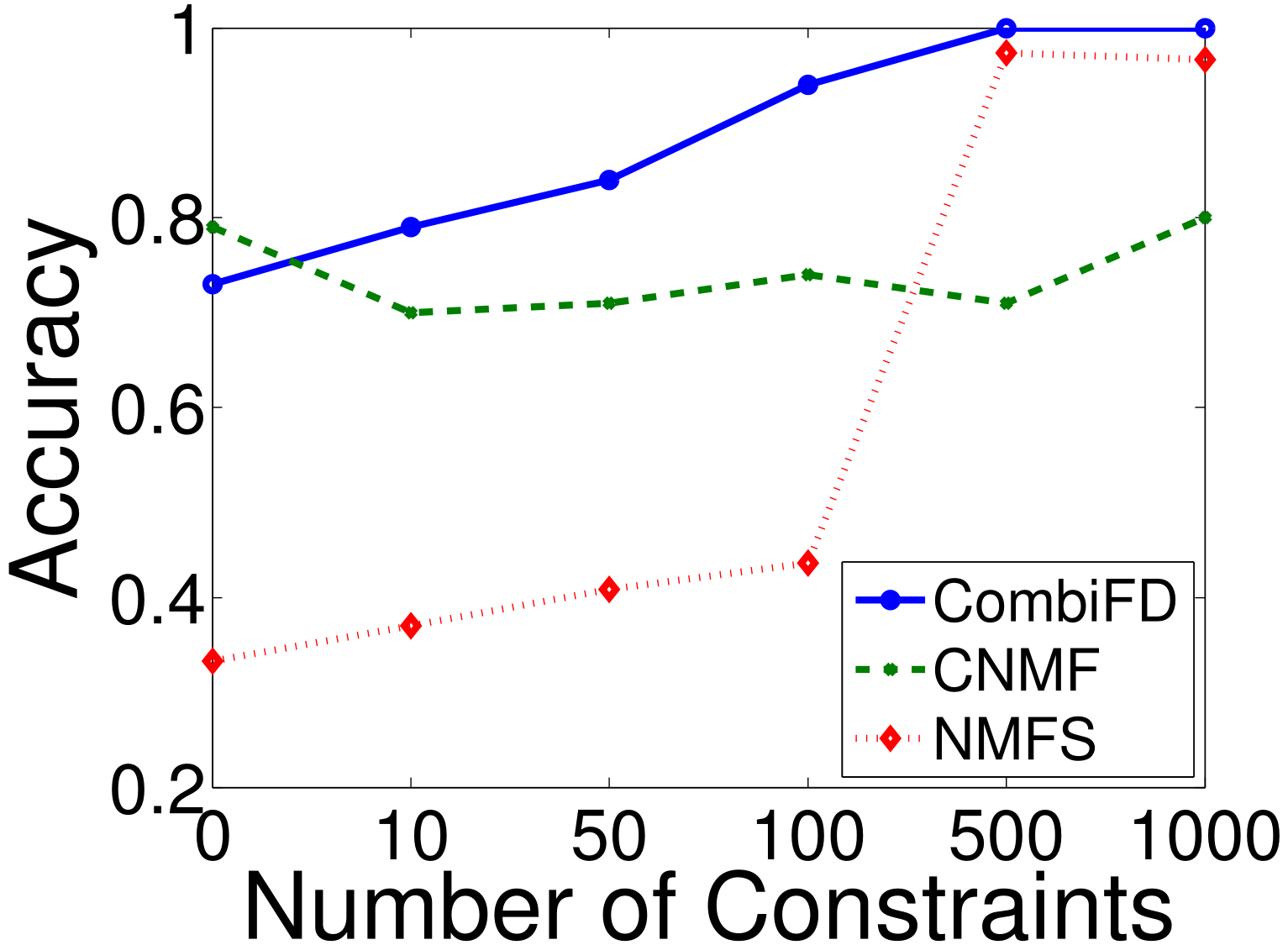}}
\hfill
    	\subfloat[Wine]{\label{grWine}\includegraphics[trim=0 0 00 0, clip, width=0.22\textwidth]{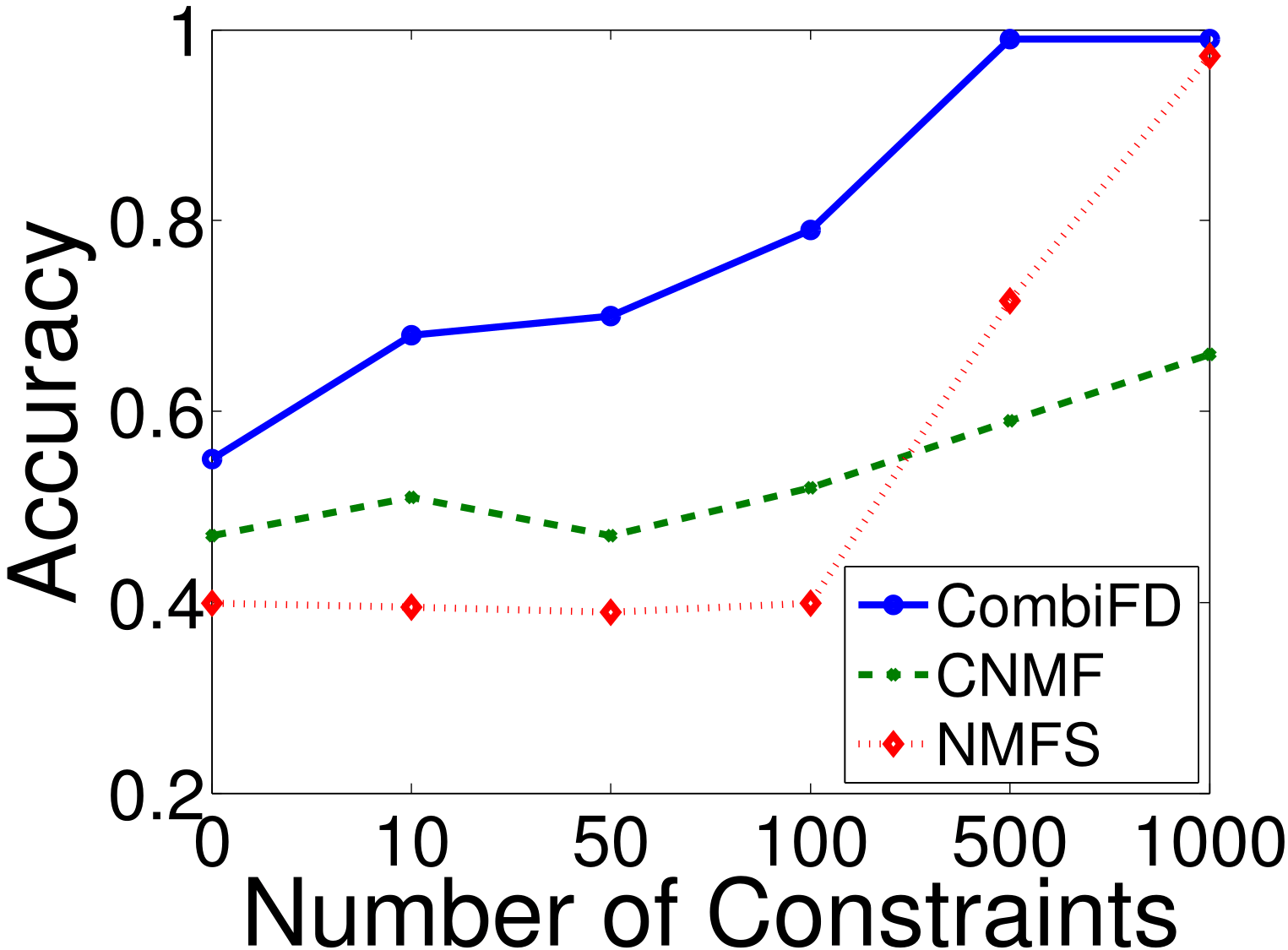}}
    	\hfill
	\subfloat[Ionosphere]{\label{grIono}\includegraphics[trim=0 0 00 0, clip, width=0.22\textwidth]{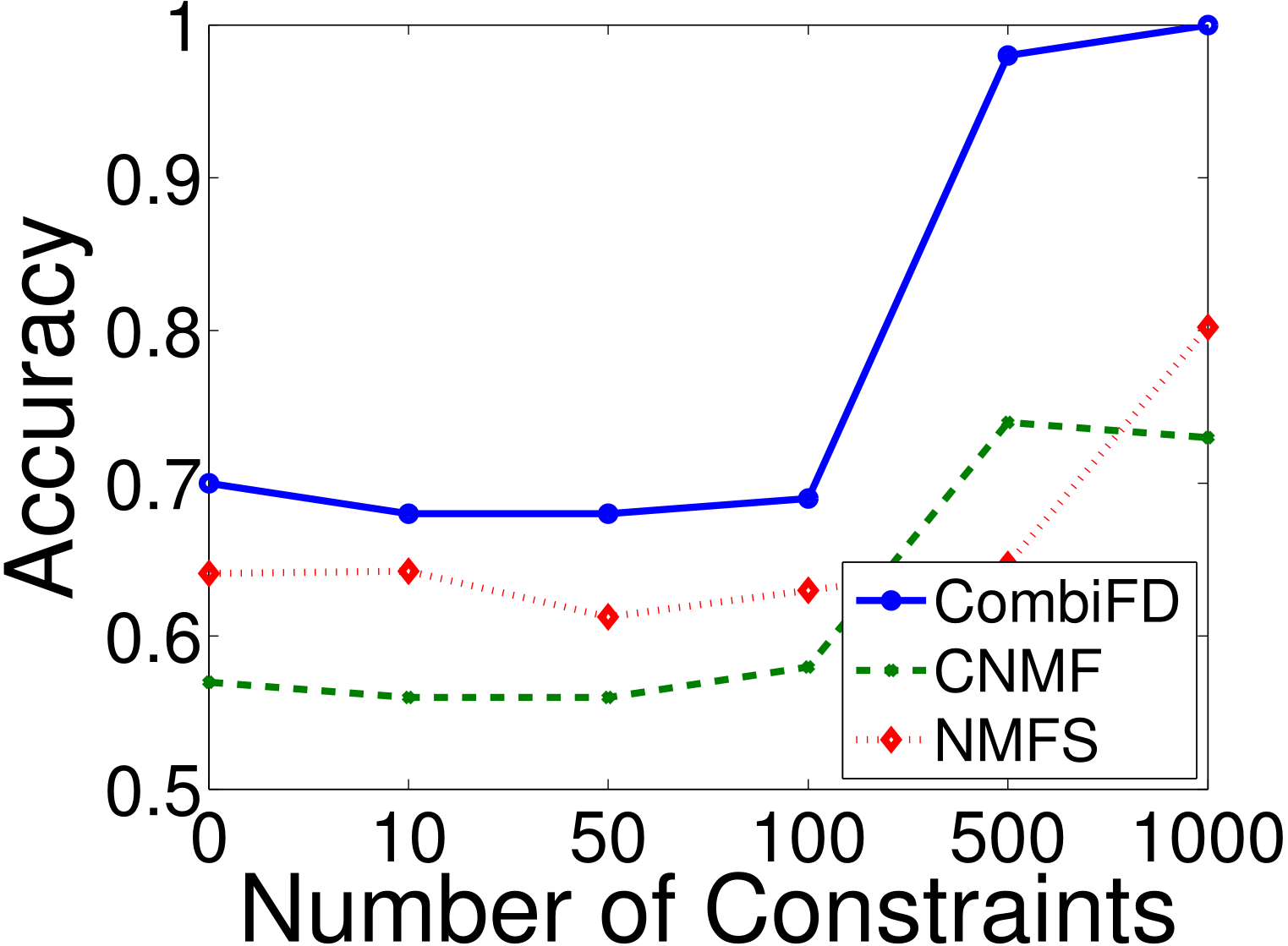}}
    	\hfill
   	\subfloat[Seeds]{\label{grSeeds}\includegraphics[trim=0 0 00 0, clip, width=0.22\textwidth]{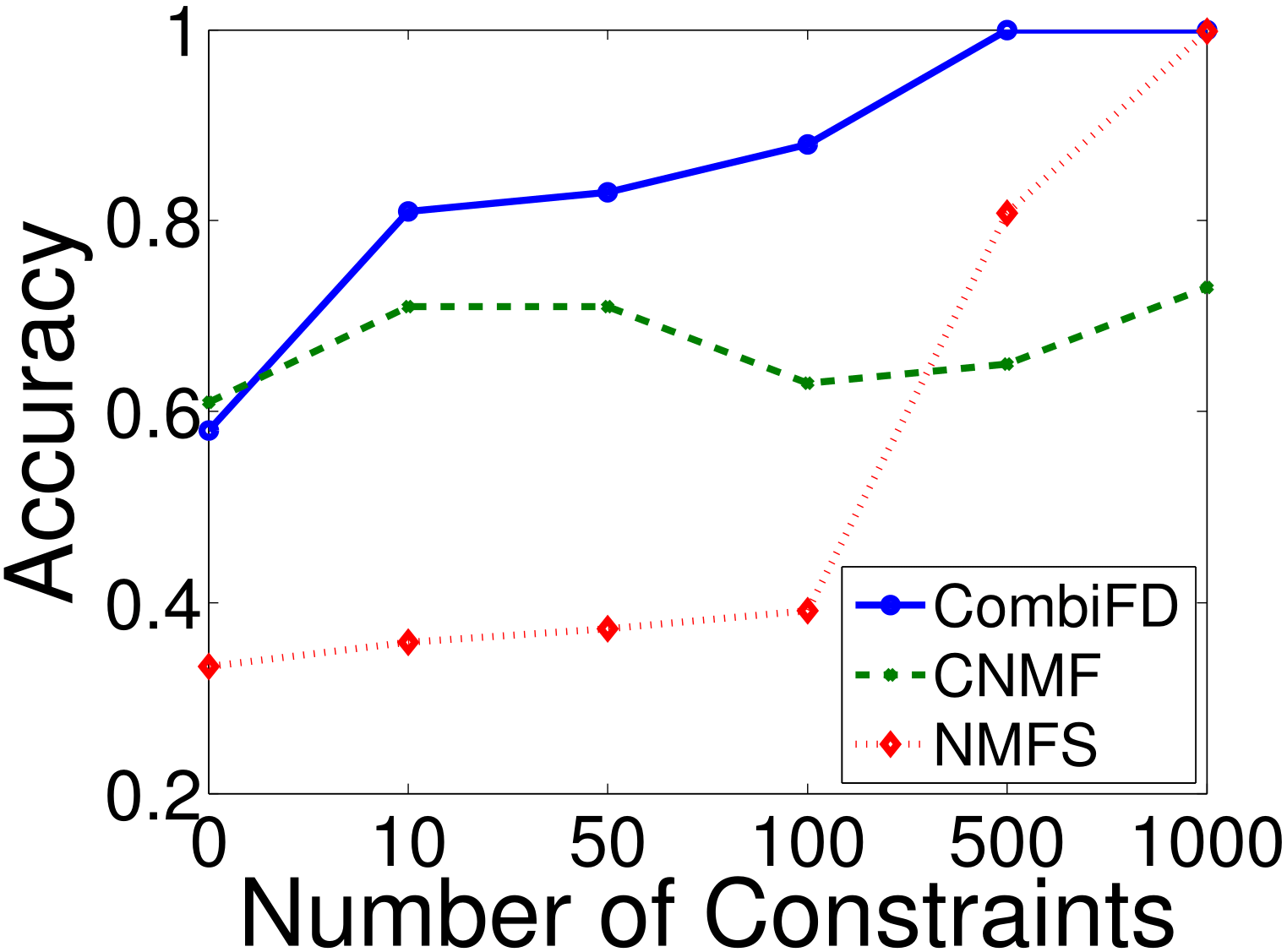}}
  	\caption{Accuracy as a function of the amount of prior knowledge. CombiFD (solid blue line) clearly outperforms competing techniques across all levels of supervision (number of constraints).}
     \label{MLCLAccuracy}
\end{figure*}

\subsection{Clustering with more complex prior knowledge}

In this experiment, we consider the Zoo dataset from UCI~\cite{BacheLichman2013}. This is a dataset of $101$ animals, each one represented as a vector of $17$ non-negative features (e.g. whether it has hair or not, whether it has feathers or not, or its number of legs). There are $7$ class labels.

We solve the clustering problem using our CombiFD approach (\ref{cnmf_def}), where we enforce non-negativity as well as additional constraints capturing some well known biological facts. Specifically, we enforce the following constraints on the basis vectors using our CombiFD framework: $\neg (HasMilk \wedge HasEggs)$, $HasFeathers \rightarrow HasEggs$, and $\neg (HasFeathers \wedge HasHair)$. Since we are not aware of any other matrix factorization technique that can take into account this kind of complex, logically structured prior knowledge, we compare with standard NMF (problem (\ref{nmf_opt})), which is a totally unsupervised technique.

\begin{small}
\begin{table}[ht]
\centering
\small
\begin{tabular}{ccccc}
			\hline
			\\[-1.5ex]
			\textbf{Approach} &  
			\multicolumn{2}{c}{
			\textbf{Accuracy}} &  
			\multicolumn{2}{c}{
			\textbf{Time (\emph{seconds})}} 
			\\[1pt]
			& Avg. & Std. Dev. & Avg. & Std. Dev.
			\\[1pt]\hline
			\\[-1.5ex]
			NMF & 0.72 & 0.08 & \textbf{1.84} & 0.1\\
			CombiFD & \textbf{0.82} & 0.05 & 4.06 & 0.8\\
			\hline
		\end{tabular}
\caption{Accuracy and runtime of NMF vs. CombiFD on the UCI Zoo dataset for $100$ runs.}
\label{tab:clustering}
\end{table}
\end{small}

Table \ref{tab:clustering} reports the accuracy (defined as before) and runtime of the two approaches averaged over $100$ runs.
The results show that the CombiFD approach, which incorporates a limited amount of logically structured prior knowledge, significantly improves the accuracy of standard NMF, while still running within seconds using AMIQO.

\subsection{Spectral Unmixing for Materials Discovery}

\begin{figure*}[t]
\centering
		\includegraphics[scale=0.20]{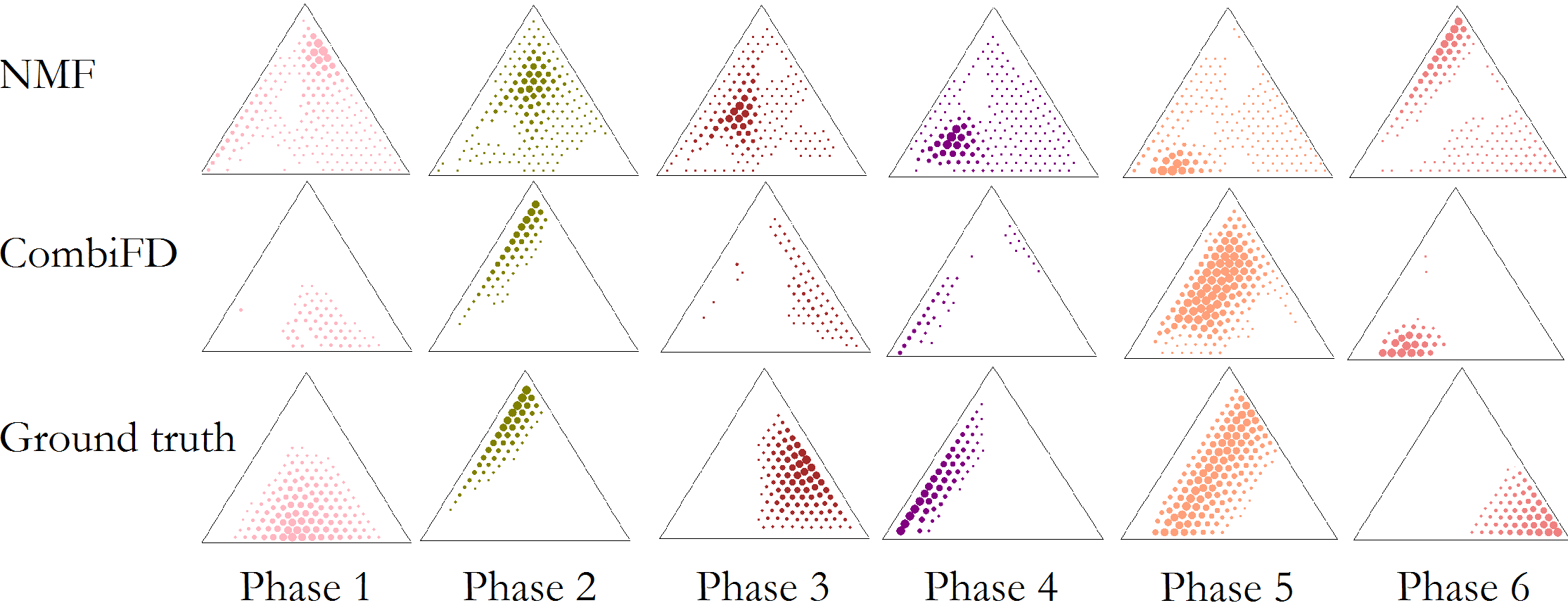}
	\caption{Results of NMF (top row), CombiFD (middle row), and ground truth (bottom row) for the Al-Li-Fe oxide system. In each row, each one of the $6$ plots corresponds to the map of the concentration of a basis pattern (crystal phase). The size of each dot is proportional to the phase concentration estimated at that point from the coefficient matrix $H$. It can be seen that standard NMF overlooks the physical constraints, e.g. there are often mixtures of more than $3$ basis patterns and the phase regions are disconnected. CombiFD recovers accurately all the phases except for the last one. }
	\label{fig:allife}
\end{figure*}

We first present how to incorporate complex constraints capturing some of the key physical laws that govern the data-generating process for the phase map identification problem.

\textbf{Sparsity: }
According to the so-called Gibbs phase rule, under equilibrium conditions at fixed temperature and pressure, there can be mixtures of at most $M$ phases occurring at each data point  (chemical composition) in a library involving $M$ basic elements. 
This is encoded as $||H_s||_0 \leq M$, for $s \in [1,k]$ as in (\ref{cnmf_sparse_def}).

\textbf{Shifting: }
Each basis pattern may be slightly stretched by a different amount for each composition sample. Indeed, chemical alloying within a given compound may alter the crystal lattice constants, leading to a systematic shift (as a function of composition) of peak positions in the measured X-ray patterns. For isotropic lattice expansion and signals measured versus the scattering vector magnitude, the peak shifts are proportional to the peak positions, corresponding to a linear stretch of the signal. Therefore, we use $Qk$ basis vectors, where $k$ are free and $(Q-1)k$ are constrained to be shifted versions of the free basis vectors. This is encoded as follows: $A_{i,j} = interpolate (A_{\lfloor i / (1+\ell \gamma) \rfloor,zQ} ,A_{\lceil i / (1+\ell \gamma) \rceil,zQ} )$ for $j=z Q + \ell$, $z \in [0,k-1]$, $\ell \in [1,Q-1]$, where $\gamma$ is a constant for the shifting granularity and $interpolate(x,y)$ denotes linear interpolation between $x$ and $y$. By choosing Q to be no smaller than the ratio of the maximum peak shift and the minimum peak width (a known value for a given experiment), a linear sum of the Q mutually-constrained basis patterns can accurately model each instance of the shifting phase pattern. 

\textbf{Connectivity: }
The compositions at which a phase (or basis) is observed should form a connected region in composition space, 
and its lattice parameters should vary smoothly across the region.
Hence we build a graph $G=(V,E)$ with $n$ vertices (one vertex per data point) and edges between sample points that have similar compositions.
Let $G(H_\ell)$ be the subgraph induced by the vertex set $V(H_\ell)=\{i: H_{\ell,i}>0\} \subseteq V$ (the set of vertices where phase $\ell$ is used). We want to enforce $G(H_\ell)$ is connected for $\ell=1,\cdots,k$. The first formulation we consider is flow based.
Intuitively, this flow-based encoding defines a flow for each phase $\ell$ that can only pass through vertices where the phase is present. There is a source node that injects positive flow in the network, and there is some outgoing flow at every vertex where a phase is used. In order to satisfy flow conservation, there has to be a path from the source to any other node belonging to the same phase in the network. This constraint enforces connectivity but it can be expensive to include in the MIQP formulation when there is a large number of points. We therefore also consider the following variation. For all triples of points $v_1,v_2,v_3$ that lie on a straight line (in this order) in the composition space, and for every phase $\ell$, we enforce the constraint that $h_{\ell,v_2} \geq \min \{h_{\ell,v_1},h_{\ell,v_3}\}$. Although these constraints do not strictly enforce connectivity, they are simpler and often strong enough that the solution obtained is actually connected. 

We consider synthetic data from \cite{LeBras2014Computational}, generated from the Aluminium(Al)-Lithium(Li)-Iron(Fe) oxide system and for which the ground truth is known. This dataset has $219$ points and $6$ underlying phases, where each phase has up to $42$ diffraction peaks. We report in Fig. \ref{fig:allife} the data interpretation obtained using regular NMF, CombiFD and the ground truth. On average, the number of iterations of AMIQO is about $8$, with runtimes ranging from minutes to a few hours. We can see that although we cannot recover exactly the ground truth solution, our interpretation obtained considering prior domain knowledge is much more accurate. 
For example, using NMF the sparsity constraint is violated for many sample points.
In terms of evaluation metric we adapt the previous definition of accuracy to the case of soft cluster assignment, to reflect the fact that a sample point might be assigned to multiple clusters. Namely, we define $AC=1/n \cdot max_{\sigma} \sum_{i=1}^{k}|r_i \cap c_{\sigma_i}|/|r_i \cup c_{\sigma_i}|$
. Note that this metric does not explicitly reflect the amount of violated combinatorial constraints, which would be interesting to evaluate in future work. We report the results in Table \ref{tab:matdiscres}. Overall, CombiFD outperforms NMF, confirming the visual that incorporating prior knowledge indeed improves the accuracy of the interpretation. 
Compared with the previous constraint programming formulation (SMT)~\cite{ermon2012smt}, the CombiFD algorithm exhibits much better scalability, allowing us to quickly analyze more realistic sized datasets with hundreds of sample points in a few hours. In fact, SMT could not find solutions after 20 hours on all instances but one.   

Finally, we show results obtained with CombiFD on a real dataset (Fe-Bi-V oxide system, Fig. \ref{fig:febiv}). While the phase map is unknown for this system, Fig. 3 shows the excellent match between the phase $1$ basis pattern from the CombiFD solution and the pattern of the known Bi$_4$V$_2$O$_{11}$ phase~\cite{bergerhoff1987crystallographic}. Chemical alloying of Fe into this compound, of the form Bi$_4$V$_{2-x}$Fe$_x$O$_{11-x}$, has also been observed, which may be related to the identification of this phase 1 basis pattern over a wide composition range in the library~\cite{vannier2003bi}. The results of Fig. 3 demonstrate the ability of CombiFD to identify well-connected phase regions from experimental data, and further investigations are underway to evaluate the presence of minority phases with weak signals in this composition library.

\begin{figure*}[t]
\centering
		\includegraphics[scale=0.29]{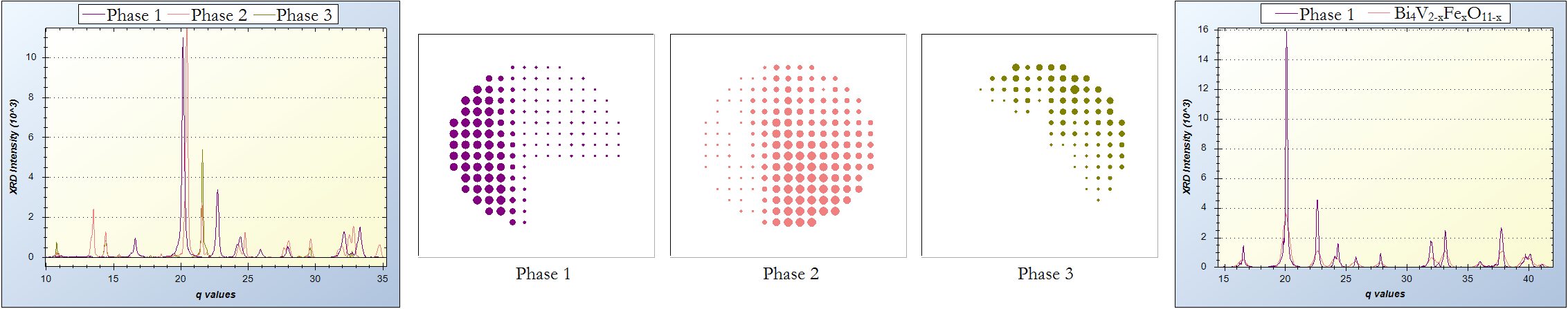}
	\caption{Results obtained with CombiFD and $3$ phases (k=3) on the Fe-Bi-V oxide system. The X-ray diffraction data was collected at the SLAC National Accelerator Laboratory. The plot on the left shows the patterns of the 3 phases for low values of q (10-30 nm$^{-1}$), while the center 3 diagrams indicate where each phase appears. The right plot shows how one (recovered) phase matches a phase from the literature. 
}
	\label{fig:febiv}
\end{figure*}

\begin{small}
\begin{table}[tb]
\centering
\small
\begin{tabular}{cc|ccc}
			\hline
			\\[-1ex]
\multicolumn{2}{c|}{Parameters} & \multicolumn{3}{c}{Accuracy}\\
 $n$ & $m$ & \textbf{NMF} & \textbf{CombiFD}  & \textbf{SMT}
\\[1pt]\hline
			\\[-1.5ex]
28&650&0.65&0.86& \textbf{0.96}\\
60&300&0.68&\textbf{0.77}& \emph{t.-o.}\\
60&650&0.66&\textbf{0.75}&\emph{t.-o.}\\
219&300&0.64&\textbf{0.73}&\emph{t.-o.}\\
219&650&0.64&\textbf{0.73}&\emph{t.-o.}\\
219&100&0.64&\textbf{0.75}&\emph{t.-o.}\\
219&200&0.63&\textbf{0.73}&\emph{t.-o.}\\

\hline
		\end{tabular}
\caption{Accuracy of NMF vs. CombiFD vs. SMT for different instances of the Al-Li-Fe oxide system. \emph{t.-o.} indicates a time-out after $20$ hours of CPU time.}
\label{tab:matdiscres}
\end{table}
\end{small}

\section{Conclusions}
\label{conclusions}
The ability to integrate complex prior knowledge into unsupervised data analysis approaches is a rich and challenging research problem, pervasive in a variety of domains, such as scientific discovery. In particular, the motivating application of this work is the discovery of new fuel cell and solar fuel materials, thereby addressing pressing issues in sustainability such as the need for renewable and clean energy.
We introduced CombiFD, a novel factor-based pattern decomposition framework that significantly generalizes and extends prior approaches. Our framework allows the specification of general constraints (including combinatorial ones), which are used to specify a-priori domain knowledge on the factorization that is being sought. These include traditional constraints such as non-negativity as well as more intricate dependencies, such as the ones coming from known phase behavior of chemical systems. We introduced a general factorization algorithm called AMIQO, based on solving a sequence of (mixed-integer) quadratic programs. We showed that AMIQO outperforms state-of-the-art approaches on a key problem in materials discovery: it scales to large real world datasets, it can handle complex, logically structured prior knowledge and by including prior knowledge into the model, we obtain significantly more accurate interpretations of the data.
There are many directions to further extend our work, namely concerning representation formalism to capture other combinatorial constraints, with good performance/runtime trade-offs, as well as other algorithms to solve the combinatorial optimization problem.

\fontsize{9.5pt}{10.5pt}

\section{Acknowledgments}
\label{acknowledgments}

Work supported by the National Science Foundation (NSF Expeditions in Computing award for Computational Sustainability, grant 0832782, NSF Inspire grant 1344201 and NSF Eager grant 1258330). Experiments were run on an infrastructure supported by the NSF Computing research infrastructure for Computational Sustainability grant 1059284. In collaboration with the Energy Materials Center at Cornell (emc2), an Energy Frontier Research Center funded by the U.S. Department of Energy (DOE), Office of Science, Office of Basic Energy Sciences (Award No. DE-SC0001086). Research conducted at the Cornell High Energy Synchrotron Source (CHESS) is supported by the NSF and the National Institutes of Health/National Institute of General Medical Sciences under NSF award DMR-0936384. Based upon work performed by the Joint Center for Artificial Photosynthesis (JCAP), a DOE Energy Innovation Hub, supported through the Office of Science of the U.S. Department of Energy (Award No. DE-SC0004993). Use of the Stanford Synchrotron Radiation Lightsource, SLAC National Accelerator Laboratory, is supported by the U.S. DOE, Office of Science, Office of Basic Energy Sciences under Contract No. DE-AC02-76SF00515.  

\appendix
\section{Appendix: Proofs}

\textbf{Proof of Proposition \ref{thProperties}}
The proof is by induction. Suppose $(W^j,H^j,x^j,b^j)$ is feasible for (\ref{cnmf_def}) (the base case $j=0$ holds by construction). 
It follows that (\ref{cnmf_def}) augmented with the additional constraint $H=H^j$ is still feasible, and therefore $\min_{W,x,b} f(W,H^{0})$ subject to (\ref{cnmf_def}) and $H=H^j$ admits an optimal solution $W^{j+1},\tilde{x}^{j+1},\tilde{b}^{j+1}$. Since $W^{j+1},H^j,\tilde{x}^{j+1},\tilde{b}^{j+1}$ is feasible for (\ref{cnmf_def}), it follows that (\ref{cnmf_def}) augmented with the additional constraint $W=W^{j+1}$ is also feasible. Therefore, $\min_{H,x,b} f(W^{j+1},H)$ subject to (\ref{cnmf_def}) and $W=W^{j+1}$ admits an optimal solution $H^{j+1},x^{j+1},b^{j+1}$. It also follows that $W^{j+1}, H^{j+1},x^{j+1},b^{j+1}$ is feasible for (\ref{cnmf_def}). Finally, since we are optimizing at every step, it follows that $|| A-W^jH^j||_{p} \geq || A-W^{j+1} H^j||_{p} \geq || A-W^{j+1}H^{j+1}||_{p}$.

\textbf{Proof of Theorem \ref{thKmeanEquiv}}
The initial (feasible) values of $H^0,b^0$ can be seen as an initial (hard) assignment of data points to clusters, where data point $i$ belongs to cluster $s$ if $b^0_{s,i}=1$.
The optimal solution for $\min_{W,b} f(W,H^{j})$ subject to (\ref{cnmf_def}) and $H=H^j$ is to choose each column $W$ to be the centroid of the data points assigned to the corresponding cluster by $b^j$ at iteration $j$, i.e. for each $s$ set the $s$-th column of $W$ to be $w_s = 1/(\sum_i b_{s,i}) \sum a_i b_{s,i}$ (non-negative because the data points are assumed to be non-negative).
An optimal solution for $\min_{H,b} f(W^{j+1},H)$ subject to (\ref{cnmf_def}) and $W=W^{j+1}$ can be found by assigning each data point $i$ to the cluster whose centroid $w_s$ is closest to data point $a_i$, i.e. setting $h_{s^\ast(i),i} = b_{s^\ast(i),i} = 1$ where $s^\ast(i)=\arg \min_s ||w_s - a_i||_2 $.
These operations exactly correspond to $k$-means clustering initialized with the hard cluster assignment given by $b^0$.

\begin{small}
\bibliographystyle{aaai}
\bibliography{bibs/MDbib,bibs/NMF,bibs/MDbib2}

\begin{thebibliography}{}

\bibitem[\protect\citeauthoryear{Albright \bgroup et al\mbox.\egroup
  }{2006}]{albright2006algorithms}
Albright, R.; Cox, J.; Duling, D.; Langville, A.; and Meyer, C.
\newblock 2006.
\newblock Algorithms, initializations, and convergence for the nonnegative
  matrix factorization.
\newblock Technical report, NCSU Tech Report.

\bibitem[\protect\citeauthoryear{Bache and Lichman}{2013}]{BacheLichman2013}
Bache, K., and Lichman, M.
\newblock 2013.
\newblock {UCI ML} repository.

\bibitem[\protect\citeauthoryear{Basu, Davidson, and
  Wagstaff}{2008}]{basu2008constrained}
Basu, S.; Davidson, I.; and Wagstaff, K.
\newblock 2008.
\newblock {\em Constrained clustering: Advances in algorithms, theory, and
  applications}.
\newblock CRC Press.

\bibitem[\protect\citeauthoryear{Bergerhoff and
  Brown}{1987}]{bergerhoff1987crystallographic}
Bergerhoff, G., and Brown, I.
\newblock 1987.
\newblock Crystallographic databases.
\newblock {\em International Union of Crystallography, Chester}  77--95.

\bibitem[\protect\citeauthoryear{Berry \bgroup et al\mbox.\egroup
  }{2007}]{berry2007algorithms}
Berry, M.~W.; Browne, M.; Langville, A.~N.; Pauca, V.~P.; and Plemmons, R.~J.
\newblock 2007.
\newblock Algorithms and applications for approximate nonnegative matrix
  factorization.
\newblock {\em Computational Statistics \& Data Analysis} 52(1):155--173.

\bibitem[\protect\citeauthoryear{Cai \bgroup et al\mbox.\egroup
  }{2011}]{cai2011graph}
Cai, D.; He, X.; Han, J.; and Huang, T.~S.
\newblock 2011.
\newblock Graph regularized nonnegative matrix factorization for data
  representation.
\newblock {\em PAMI, IEEE Transactions on} 33(8):1548--1560.

\bibitem[\protect\citeauthoryear{Choo \bgroup et al\mbox.\egroup
  }{2013}]{choo2013utopian}
Choo, J.; Lee, C.; Reddy, C.~K.; and Park, H.
\newblock 2013.
\newblock Utopian: User-driven topic modeling based on interactive nonnegative
  matrix factorization.
\newblock {\em Visualization and Computer Graphics, IEEE Transactions on}
  19(12):1992--2001.

\bibitem[\protect\citeauthoryear{Ding, Li, and Jordan}{2010}]{ding2010convex}
Ding, C.~H.; Li, T.; and Jordan, M.~I.
\newblock 2010.
\newblock Convex and semi-nonnegative matrix factorizations.
\newblock {\em Pattern Analysis and Machine Intelligence, IEEE Transactions.}
  32(1):45--55.

\bibitem[\protect\citeauthoryear{Ding, Li, and
  Peng}{2008}]{ding2008equivalence}
Ding, C.; Li, T.; and Peng, W.
\newblock 2008.
\newblock On the equivalence between non-negative matrix factorization and
  probabilistic latent semantic indexing.
\newblock {\em Computational Stat. \& Data Analysis} 52(8):3913--3927.

\bibitem[\protect\citeauthoryear{Ermon \bgroup et al\mbox.\egroup
  }{2012}]{ermon2012smt}
Ermon, S.; Le~Bras, R.; Gomes, C.~P.; Selman, B.; and van Dover, R.~B.
\newblock 2012.
\newblock Smt-aided combinatorial materials discovery.
\newblock In {\em SAT 2012},  172--185.
\newblock Springer.

\bibitem[\protect\citeauthoryear{Ginley \bgroup et al\mbox.\egroup
  }{2005}]{ginley}
Ginley, D.; Teplin, C.; Taylor, M.; van Hest, M.; and Perkins, J.
\newblock 2005.
\newblock Combinatorial materials science.
\newblock In {\em AccessScience}.
\newblock McGraw-Hill Companies.

\bibitem[\protect\citeauthoryear{Gomes}{2009}]{gomes2009computational}
Gomes, C.~P.
\newblock 2009.
\newblock Computational sustainability: Computational methods for a sustainable
  environment, economy, and society.
\newblock {\em The Bridge} 39(4):5--13.

\bibitem[\protect\citeauthoryear{Gregoire \bgroup et al\mbox.\egroup
  }{2010}]{gregoireptta}
Gregoire, J.~M.; Tague, M.~E.; Cahen, S.; Khan, S.; Abruna, H.~D.; DiSalvo,
  F.~J.; and van Dover, R.~B.
\newblock 2010.
\newblock Improved fuel cell oxidation catalysis in pt1-xtax.
\newblock {\em Chem. Mater.} 22(3):1080.

\bibitem[\protect\citeauthoryear{Haber \bgroup et al\mbox.\egroup
  }{2014}]{Haber2014Discovering}
Haber, J.~A.; Cai, Y.; Jung, S.; Xiang, C.; Mitrovic, S.; Jin, J.; Bell, A.~T.;
  and Gregoire, J.~M.
\newblock 2014.
\newblock Discovering ce-rich oxygen evolution catalysts{,} from high
  throughput screening to water electrolysis.
\newblock {\em Energy Environ. Sci.} 7:682--688.

\bibitem[\protect\citeauthoryear{Halevy, Norvig, and
  Pereira}{2009}]{halevy2009unreasonable}
Halevy, A.; Norvig, P.; and Pereira, F.
\newblock 2009.
\newblock The unreasonable effectiveness of data.
\newblock {\em Intelligent Systems, IEEE} 24(2):8--12.

\bibitem[\protect\citeauthoryear{Hossain \bgroup et al\mbox.\egroup
  }{2010}]{hossain2010unifying}
Hossain, M.~S.; Tadepalli, S.; Watson, L.~T.; Davidson, I.; Helm, R.~F.; and
  Ramakrishnan, N.
\newblock 2010.
\newblock Unifying dependent clustering and disparate clustering for
  non-homogeneous data.
\newblock In {\em SIGKDD},  593--602.
\newblock ACM.

\bibitem[\protect\citeauthoryear{Hoyer}{2004}]{hoyer2004non}
Hoyer, P.~O.
\newblock 2004.
\newblock Non-negative matrix factorization with sparseness constraints.
\newblock {\em Journal of ML Research}  1457--1469.

\bibitem[\protect\citeauthoryear{Kim and Park}{2008}]{kim2008nonnegative}
Kim, H., and Park, H.
\newblock 2008.
\newblock Nonnegative matrix factorization based on alternating nonnegativity
  constrained least squares and active set method.
\newblock {\em SIAM} 30(2):713--730.

\bibitem[\protect\citeauthoryear{Kim, Sra, and Dhillon}{2007}]{kim2007fast}
Kim, D.; Sra, S.; and Dhillon, I.~S.
\newblock 2007.
\newblock Fast newton-type methods for the least squares nonneg. matrix approx.
  problem.
\newblock In {\em SDM}.

\bibitem[\protect\citeauthoryear{Kuang, Park, and
  Ding}{2012}]{kuang2012symmetric}
Kuang, D.; Park, H.; and Ding, C.~H.
\newblock 2012.
\newblock Symmetric nonnegative matrix factorization for graph clustering.
\newblock In {\em SDM},  106--117.

\bibitem[\protect\citeauthoryear{Kusne \bgroup et al\mbox.\egroup
  }{2014}]{kusne2014fly}
Kusne, A.~G.; Gao, T.; Mehta, A.; Ke, L.; Nguyen, M.~C.; Ho, K.-M.; Antropov,
  V.; Wang, C.-Z.; Kramer, M.~J.; Long, C.; et~al.
\newblock 2014.
\newblock On-the-fly machine-learning for high-throughput experiments: search
  for rare-earth-free permanent magnets.
\newblock {\em Scientific reports} 4.

\bibitem[\protect\citeauthoryear{{Le Bras} \bgroup et al\mbox.\egroup
  }{2011}]{lebras2011constraint}
{Le Bras}, R.; Damoulas, T.; Gregoire, J.~M.; Sabharwal, A.; Gomes, C.~P.; and
  Van~Dover, R.~B.
\newblock 2011.
\newblock Constraint reasoning and kernel clustering for pattern decomposition
  with scaling.
\newblock In {\em CP}.
\newblock  508--522.

\bibitem[\protect\citeauthoryear{{Le Bras} \bgroup et al\mbox.\egroup
  }{2014}]{LeBras2014Computational}
{Le Bras}, R.; Bernstein, R.; Gregoire, J.~M.; Suram, S.~K.; Gomes, C.~P.;
  Selman, B.; and {van Dover}, R.~B.
\newblock 2014.
\newblock A computational challenge problem in materials discovery: Synthetic
  problem generator and real-world datasets.
\newblock In {\em AAAI}.

\bibitem[\protect\citeauthoryear{Lee and Seung}{1999}]{lee1999learning}
Lee, D.~D., and Seung, H.~S.
\newblock 1999.
\newblock Learning the parts of objects by non-negative matrix factorization.
\newblock {\em Nature} 401(6755):788--791.

\bibitem[\protect\citeauthoryear{Li, Ding, and Jordan}{2007}]{li2007solving}
Li, T.; Ding, C.; and Jordan, M.~I.
\newblock 2007.
\newblock Solving consensus and semi-supervised clustering problems using
  nonnegative matrix factorization.
\newblock In {\em Data Mining, 2007. ICDM 2007. Seventh IEEE International
  Conference on},  577--582.
\newblock IEEE.

\bibitem[\protect\citeauthoryear{Lin}{2007}]{lin2007projected}
Lin, C.-J.
\newblock 2007.
\newblock Projected gradient methods for nonnegative matrix factorization.
\newblock {\em Neural comp.} 19(10):2756--2779.

\bibitem[\protect\citeauthoryear{Liu and Wu}{2010}]{liu2010non}
Liu, H., and Wu, Z.
\newblock 2010.
\newblock Non-negative matrix factorization with constraints.
\newblock In {\em Twenty-Fourth AAAI Conference}.

\bibitem[\protect\citeauthoryear{Long \bgroup et al\mbox.\egroup
  }{2009}]{long09}
Long, C.; Bunker, D.; Karen, V.; Li, X.; and Takeuchi, I.
\newblock 2009.
\newblock Rapid identification of structural phases in combinatorial thin-film
  libraries using x-ray diffraction and non-negative matrix factorization.
\newblock {\em Rev. Sci. Instruments} 80.

\bibitem[\protect\citeauthoryear{Narasimhan, Mallapragada, and
  Porter}{2007}]{narasimhan2007combinatorial}
Narasimhan, B.; Mallapragada, S.; and Porter, M.
\newblock 2007.
\newblock {\em Combinatorial materials science}.
\newblock John Wiley and Sons.

\bibitem[\protect\citeauthoryear{Paatero and
  Tapper}{1994}]{paatero1994positive}
Paatero, P., and Tapper, U.
\newblock 1994.
\newblock Positive matrix factorization: A non-negative factor model with
  optimal utilization of error estimates of data values.
\newblock {\em Environmetrics} 5(2):111--126.

\bibitem[\protect\citeauthoryear{Shahnaz \bgroup et al\mbox.\egroup
  }{2006}]{shahnaz2006document}
Shahnaz, F.; Berry, M.~W.; Pauca, V.~P.; and Plemmons, R.~J.
\newblock 2006.
\newblock Document clustering using nonnegative matrix factorization.
\newblock {\em Information Processing \& Management}  373--386.

\bibitem[\protect\citeauthoryear{Smaragdis, Raj, and
  Shashanka}{2008}]{smaragdis2008sparse}
Smaragdis, P.; Raj, B.; and Shashanka, M.~V.
\newblock 2008.
\newblock Sparse and shift-invariant feature extraction from non-negative data.
\newblock In {\em ICASSP},  2069--2072.

\bibitem[\protect\citeauthoryear{Smaragdis}{2004}]{smaragdis2004non}
Smaragdis, P.
\newblock 2004.
\newblock Non-negative matrix factor deconvolution; extraction of multiple
  sound sources from monophonic inputs.
\newblock In {\em Indep. Component Analysis and Blind Signal Separation}.
\newblock  494--499.

\bibitem[\protect\citeauthoryear{Takeuchi, {van Dover}, and
  Koinuma}{2002}]{takeuchi2002combinatorial}
Takeuchi, I.; {van Dover}, R.~B.; and Koinuma, H.
\newblock 2002.
\newblock Combinatorial synthesis and evaluation of functional inorganic
  materials using thin-film techniques.
\newblock {\em MRS bulletin} 27(04):301--308.

\bibitem[\protect\citeauthoryear{Tjioe \bgroup et al\mbox.\egroup
  }{2008}]{tjioe2008using}
Tjioe, E.; Berry, M.; Homayouni, R.; and Heinrich, K.
\newblock 2008.
\newblock Using a literature-based nmf model for discovering gene functional
  relationships.
\newblock {\em BMC Bioinformatics} 9.

\bibitem[\protect\citeauthoryear{van Dover, Schneemeyer, and
  Fleming}{1998}]{van1998discovery}
van Dover, R.~B.; Schneemeyer, L.; and Fleming, R.
\newblock 1998.
\newblock Discovery of a useful thin-film dielectric using a composition-spread
  approach.
\newblock {\em Nature} 392(6672):162--164.

\bibitem[\protect\citeauthoryear{Vannier \bgroup et al\mbox.\egroup
  }{2003}]{vannier2003bi}
Vannier, R.; Pernot, E.; Anne, M.; Isnard, O.; Nowogrocki, G.; and Mairesse, G.
\newblock 2003.
\newblock {Bi}$_4${V}$_2${O}$_{11}$ polymorph crystal structures related to
  their electrical properties.
\newblock {\em Solid State Ionics} 157(1):147--153.

\bibitem[\protect\citeauthoryear{White}{2012}]{white2012materials}
White, A.
\newblock 2012.
\newblock The materials genome initiative: One year on.
\newblock {\em MRS Bulletin} 37(08):715--716.

\bibitem[\protect\citeauthoryear{Zhi \bgroup et al\mbox.\egroup
  }{2013}]{zhi2013clustering}
Zhi, W.; Wang, X.; Qian, B.; Butler, P.; Ramakrishnan, N.; and Davidson, I.
\newblock 2013.
\newblock Clustering with complex constraints-algorithms and applications.
\newblock In {\em AAAI}.

\end{thebibliography}
\end{small}
%

\end{document}